\documentclass[11pt]{article}
\usepackage[utf8]{inputenc}
\usepackage{algorithm}
\usepackage{algorithmic}
\usepackage{amsthm}
\usepackage{amssymb}
\usepackage{bbm}
\usepackage{amsmath}
\usepackage{graphicx}
\usepackage{xcolor}
\usepackage{hhline}
\usepackage{hyperref}
\usepackage{breakcites}

\advance \oddsidemargin by -0.8in
\newcommand{\myparskip}{3pt}
\parskip \myparskip

\title{On the Computational Efficiency of \\
Adaptive and Dynamic Regret Minimization}

\usepackage{amsmath,amsfonts,amssymb}
\usepackage{mathtools}
\usepackage{amsthm} 
\usepackage{latexsym}
\usepackage{relsize}

\usepackage[capitalise]{cleveref}
\usepackage{xcolor}
\usepackage{dsfont}

\newcommand{\A}{\mathcal{A}}

\newcommand{\K}{\ensuremath{\mathcal K}}

\def\regret{\mbox{{Regret}}}

\newcommand{\ignore}[1]{}

\newcommand{\email}[1]{\texttt{#1}}

\theoremstyle{plain}
\newtheorem{theorem}{Theorem}
\newtheorem{lemma}[theorem]{Lemma}
\newtheorem{corollary}[theorem]{Corollary}
\newtheorem{proposition}[theorem]{Proposition}

\newtheorem{assumption}{Assumption}

\newtheorem*{theorem*}{Theorem}
\newtheorem*{lemma*}{Lemma}
\newtheorem*{corollary*}{Corollary}
\newtheorem*{proposition*}{Proposition}
\newtheorem*{claim*}{Claim}
\newtheorem*{fact*}{Fact}
\newtheorem*{observation*}{Observation}
\newtheorem*{assumption*}{Assumption}

\theoremstyle{definition}
\newtheorem{definition}[theorem]{Definition}

\newtheorem*{definition*}{Definition}
\newtheorem*{remark*}{Remark}
\newtheorem*{example*}{Example}

 \theoremstyle{plain}
\newtheorem*{theoremaux}{\theoremauxref}
\gdef\theoremauxref{1}

\DeclareMathAlphabet{\mathbfsf}{\encodingdefault}{\sfdefault}{bx}{n}

\def\mA{{\mathcal A}}

\renewcommand{\O}{O}

\newcommand{\E}{\mathbb{E}}

\newcommand{\reals}{\mathbb{R}}
\newcommand{\eps}{\varepsilon}

\let\oldtfrac\tfrac
\renewcommand{\tfrac}[2]{\smash{\oldtfrac{#1}{#2}}}

\let\nablaold\nabla
\renewcommand{\nabla}{\nablaold\mkern-2.5mu}

\author{
  Zhou Lu\thanks{Google AI Princeton} \thanks{Princeton University}\\
  \and
  Elad Hazan\footnotemark[1] \footnotemark[2]\\
}
\begin{document}

\maketitle

\begin{abstract}

   In online convex optimization, the player aims to minimize regret, or the difference between her loss and that of the best fixed decision in hindsight over the entire repeated game.  Algorithms that minimize (standard) regret may converge to a fixed decision, which is undesirable in changing or dynamic environments. This motivates the stronger metrics of performance, notably adaptive and dynamic regret. Adaptive regret is the maximum regret over any continuous sub-interval in time. Dynamic regret is the difference between the total cost and that of the best sequence of decisions in hindsight.  

    State-of-the-art performance in both adaptive and dynamic regret minimization suffers a computational penalty - typically on the order of a multiplicative factor that grows logarithmically in the number of game iterations. In this paper we show how to reduce this computational penalty to be doubly logarithmic in the number of game iterations, and retain near optimal adaptive and dynamic regret bounds. 


\end{abstract}

\section{Introduction}

Online convex optimization is a standard framework for iterative decision making that has been extensively studied and applied to numerous learning settings. In this setting, a player iteratively chooses a point from a convex decision set, and receives loss from an adversarially chosen loss function. Her aim is to minimize her regret, or the difference between her accumulated loss and that of the best fixed comparator in hindsight. 
However, in changing environments regret is not the correct metric, as it incentivizes static behavior \cite{hazan2009efficient}.  

There are two main directions in the literature for online learning in changing environments. 
The early work of \cite{zinkevich2003online} proposed the metric of dynamic regret, which measures the regret vs. the best changing comparator,
$$ \text{D-Regret}(\mA) = \sum_t \ell_t(x_t) - \min_{x_{1:T}^*} \ell_t(x_t^*) . $$
In general, this metric can be linear with the number of game iterations and thus vacuous. However, the dynamic regret can be sublinear, and is usually related to the path length of the  comparator, i.e. 
$$ \mathcal{P} = \sum_t \|x_t^* - x_{t+1}^* \| . $$

An alternative to dynamic regret is adaptive regret, which was proposed in \cite{hazan2009efficient}, a metric closely related to regret in the shifting-experts problem \cite{herbster1998tracking}. Adaptive regret is the maximum regret over any continuous sub-interval in time. This notion has led to algorithmic innovations that yield optimal adaptive regret bounds as well as the best known dynamic regret bounds. 

The basic technique underlying the state-of-the-art methods in dynamic online learning is based on maintaining a set of expert algorithms that have different history lengths, or attention, in consideration. Each expert is a standard regret minimization algorithm, and they are formed into a committee by a version of the multiplicative update method. This methodology is generaly known as Follow-the-Leading-History (FLH) \cite{hazan2009efficient}. It has yielded near-optimal adaptive regret, strongly-adaptive algorithms \cite{daniely2015strongly}, and near-optimal dynamic regret \cite{baby2021optimal} in a variety of settings. 

However,  all previous approaches introduce a significant computational overhead to derive adaptive or dynamic regret bounds.  The technical reasoning is that all previous approaches follow the  method of reduction of FLH, from regret to adaptive regret via expert algorithms. The best known bound on number of experts required to maintain optimal adaptive regret is  $\Theta(\log T)$. Since optimal adaptive and dynamic regret bounds are known, the {\bf main open problem}  in dynamic online learning is  improving the running time overhead.  
This is exactly the question we study in this paper, namely:

\bigskip
\noindent\fbox{\begin{minipage}{\dimexpr\textwidth-2\fboxsep-2\fboxrule\relax}
\centering
Can we improve the computational complexity of adaptive and dynamic regret\\ minimization algorithms for online convex optimization?
\end{minipage}}
\bigskip

Our main result is an exponential reduction in the number of experts required for the optimal adaptive and dynamic regret bounds.  We prove that $O(\log \log T)$ experts are sufficient to obtain near-optimal bounds for general online convex optimization. 


\subsection{Summary of Results}

Our starting point is the approach of \cite{hazan2009efficient} for minimizing adaptive regret: an expert algorithm is applied such that every expert is a (standard) regret minimization algorithm, whose starting point in time differentiates it from the other experts. Instead of restarting an expert every single iteration, previous approaches retain a set of {\it active experts}, and update only these. 

In this paper we study how to maintain this set of active experts. Previous approaches require a set size that is logarithmic in the total number of iterations. We show a trade-off between the regret bound and the number of experts needed. 
By reducing the number of active experts to $O( \frac{\log \log T}{\eps})$, we give an algorithm with an $\tilde{O}(|I|^{\frac{1+\epsilon}{2}})$ adaptive regret. This result improves upon the previous $O(\log T)$ bound, and implies more efficient dynamic regret algorithms as well: for exp-concave and strongly-convex loss, our algorithm achieves $\tilde{O}(\frac{T^{\frac{1}{3}+\epsilon} \mathcal{P}^{\frac{2}{3}}}{\epsilon})$ dynamic regret bounds, using only $O( \frac{\log \log T}{\epsilon})$ experts. 


\begin{table}[ht]
\begin{center}
\begin{tabular}{|c|c|c|}
\hline
Algorithm &  
Regret over $I=[s,t]$  &
Computation 
\\
\hline
\cite{hazan2009efficient}  & $\tilde{O}(\sqrt{T })$ & $\Theta( \log T)$ \\
\hline
\cite{daniely2015strongly}, \cite{jun2017improved} &  $\tilde{O}(\sqrt{|I|})$& $\Theta( \log T)$ \\
\hline
\cite{cutkosky2020parameter}   &  $\tilde{O}(\sqrt{\sum_{\tau=s}^t \|\nabla_{\tau}\|^2 })$& $\Theta( \log T)$ \\
\hline
\cite{lu2022adaptive} &    $ \tilde{O}(\min_H \sqrt{\sum_{\tau=s}^t \|\nabla_{\tau}\|_{H}^{*2} })$ & $\Theta(\log T)$ \\
\hline
This paper &$\tilde{O}(\sqrt{|I|^{1+\epsilon}})$ & $O(\log \log T/\epsilon)$ \\
\hline

\end{tabular}
\caption{Comparison of results on adaptive regret. We evaluate the regret performance of the algorithms on any interval $I=[s,t]$, and the $\tilde{O}$ notation hides other parameters and logarithmic dependence on horizon.}

\label{table:result_summary}
\end{center}
\end{table}

\begin{table}[ht]
\begin{center}
\begin{tabular}{|c|c|c|c|}
\hline
Algorithm &  
Loss Class &
Dynamic Regret &
Computation
\\
\hline
\cite{zinkevich2003online} & General Convex  & $O(\sqrt{\mathcal{P} T })$ & $O(1)$ \\
\hline
\cite{baby2021optimal}  & Exp-concave  & $\tilde{O}(T^{\frac{1}{3}} \mathcal{P}^{\frac{2}{3}})$ & $O(\log T)$ \\
\hline
\cite{baby2022optimal}  & Strongly-convex  & $\tilde{O}(T^{\frac{1}{3}}  \mathcal{P}^{\frac{2}{3}})$ & $O(\log T)$ \\
\hline
This Paper  & Exp-concave/Strongly-convex & $\tilde{O}(T^{\frac{1}{3}+\epsilon} \mathcal{P}^{\frac{2}{3}}/\epsilon)$ & $O(\frac{\log \log T}{\epsilon})$ \\
\hline

\end{tabular}
\caption{Comparison of results on dynamic regret.}

\label{table:result_summary dynamic}
\end{center}
\end{table}

\subsection{Related Works}

For an in-depth treatment of the framework of online convex optimization see \cite{hazan2016introduction}. 

\paragraph{Shifting experts and adaptive regret.}
Online learning with shifting experts were studied in the seminal work of \cite{herbster1998tracking}, and later \cite{bousquet2002tracking}. In this setting, the comparator is allowed to shift $k$ times between the experts, and the regret is no longer with respect to a static expert, but to a $k$-partition of $[1,T]$ in which each segment has its own expert. The algorithm Fixed-Share proposed by \cite{herbster1998tracking} is a variant of the Hedge algorithm \cite{freund1997decision}. On top of the multiplicative updates, it adds a
uniform exploration term to avoid the weight of any expert
from becoming too small. This provably allows a regret bound that tracks the
best expert in any interval. \cite{bousquet2002tracking} improved this method by mixing only with the past posteriors instead of all experts.

The optimal bounds for shifting experts apply to high dimensional continuous sets and structured decision problems and do not necessarily yield efficient algorithms. This is the motivation for adaptive regret algorithms for online convex optimization  \cite{hazan2009efficient} which gave an algorithm called Follow-the-Leading-History with $O(\log^2 T)$ adaptive regret for strongly convex online convex optimization, based on the construction of experts with exponential look-back. However, their bound on the adaptive regret for general convex cost functions was $O( \sqrt{T} \log T)$. Later, \cite{daniely2015strongly} followed this idea and generalized adaptive regret to an universal bound for any sub-interval with the same length. They obtained an improved $O(\sqrt{|I|}\log T  )$ regret bound for any interval $I$. This bound was further improved to $O(\sqrt{|I| \log T  })$ by \cite{jun2017improved} using a coin-betting technique. Recently, \cite{cutkosky2020parameter} achieved a more refined second-order bound $\tilde{O}(\sqrt{\sum_{t\in I}\|\nabla_{t}\|^2})$, and \cite{lu2022adaptive} further improved it to $\tilde{O}(\min_{H \succeq 0, Tr(H)\le d}  \sqrt{\sum_{t\in I} \nabla_{t}^{\top} H^{-1}  \nabla_{t}})$, which matches the  regret of Adagrad \cite{duchi2011adaptive}. However, these algorithms are all based on the initial exponential-lookback technique  and require $\Theta(\log T)$ experts per round, increasing the computational complexity of the base algorithm in their reduction by this factor. 

\paragraph{Dynamic regret minimization.}
The notion of dynamic regret was introduced by \cite{zinkevich2003online}, and allows the comparator to be time-varying with a bounded total movement. The work of \cite{zinkevich2003online} gave an algorithm with an $O(\sqrt{T \mathcal{P}})$ dynamic regret bound where $\mathcal{P}$ denotes the total sequential distance of the moving predictors, also called the ``path length". Although this bound is optimal in general, recently works study improvements of dynamic regret bounds under further assumptions \cite{zhao2020dynamic, zhang2017improved}. In particular, \cite{baby2021optimal, baby2022optimal} achieved an improved $\tilde{O}(T^{\frac{1}{3}}\mathcal{P}^{\frac{2}{3}})$ dynamic regret bound for exp-concave and strongyly-convex online learning, with a matching lower bound $\Omega(T^{\frac{1}{3}}\mathcal{P}^{\frac{2}{3}})$.

Another line of work explores the relationship between these two metrics, and show that adaptive regret implies dynamic regret \cite{zhang2018dynamic}. \cite{zhang2020minimizing} gave algorithms that achieve both adaptive and dynamic regrets simultaneously.

\paragraph{Parameter-free online convex optimizaton.}
Related to adaptivity, an important building block in adaptive algorithms to attain tighter bounds are parameter-free online learning initiated in \cite{mcmahan2010adaptive}. Later parameter-free methods \cite{luo2015achieving, orabona2016coin, cutkosky2020parameter} attained the optimal $\tilde{O}(GD\sqrt{T})$ regret for online convex optimization without knowing any constants ahead of time, and without the usual logarithmic penalty that is a consequence of the  doubling trick.

\paragraph{Applications of adaptive online learning.}
Efficient adaptive and dynamic regret algorithms have implications in many other areas. 
A recent example is the field of online control \cite{hazan2022introduction}. The work of \cite{baby2021optimal}, \cite{baby2022optimal} used adaptive regret algorithms as building-blocks to derive tighter dynamic regret bounds. In this variant of differentiable reinforcement learning, online learning is used to generate iterative control signals, mostly for linear dynamical systems. Recent work by \cite{gradu2020adaptive,minasyan2021online} considered smooth dynamical systems, and their Lyapunov linearization. They use adaptive and dynamic regret algorithms to obtain provable bounds for time-varying systems. Thus, our results imply more efficient algorithms for control. 

Other applications of adaptive algorithms are in the area of time series prediction \cite{koolen2015minimax} and mathematical  optimization \cite{lu2022adaptive}. Our improved computational efficiency for adaptive and dynamic regret implies faster algorithms for these applications as well.

\subsection{Paper Outline}
In Section \ref{s2}, we formally define the online convex optimization framework and the basic assumptions we need. In Section \ref{s3}, we present our algorithm and show a simplified analysis that leads to an $\tilde{O}(|I|^{\frac{3}{4}})$ adaptive regret bound with doubly-logarithmic number of experts. We generalize this analysis and give our main theoretical guarantee in Section 
\ref{s4}. The limit of the FLH framework recursion is discussed in Section \ref{s6}. Implication to more efficient dynamic regret algorithms is presented in Section \ref{s5}.


\section{Setting}\label{s2}
We consider the online convex optimization (OCO) problem. At each round $t$, the player $\A$ chooses $x_t\in \K$ where $\K\subset \reals^d$ is some convex domain. The adversary then reveals loss function $\ell_t(x)$, and the player suffers loss $\ell_t(x_t)$. The goal is to minimize regret:
$$
\regret(\A)=\sum_{t=1}^T \ell_t(x_t)-\min_{x\in \K} \sum_{t=1}^T \ell_t(x) .
$$

A more subtle goal is to minimize the regret over different sub-intervals of $[1,T]$ at the same time, corresponding to a potential changing environment, which is captured by the notion of adaptive regret introduced by \cite{hazan2009efficient}. \cite{daniely2015strongly} extended this notion to depend on the length of sub-intervals, and provided an algorithm that achieves an $\tilde{O}(\sqrt{|I|})$ regret bound for all sub-intervals $I$. In particular, they define strongly adaptive regret as follows:
$$
\text{SA-Regret}(\A,k)=\max_{I=[s,t],t-s=k}\left( \sum_{\tau=s}^t \ell_{\tau}(x_{\tau})-\min_{x\in \K} \sum_{\tau=s}^{t} \ell_{\tau}(x)\right) .
$$

We make the following assumption on the loss $\ell_t$ and domain $\K$, which is standard in literature. 
\begin{assumption}
    Loss $\ell_t$ is convex, $G$-Lipschitz and non-negative. The domain $\K$ has diameter $D$. 
\end{assumption}

We define the path length $\mathcal{P}$ of a dynamic comparator $\{x_t^*\}$ w.r.t. some norm $\| \cdot \|$.
$$
\mathcal{P} = \sum_t \|x_t^* - x_{t+1}^* \| 
$$

We also define strongly-convex and exp-concave functions here.

\begin{definition}
A function $f(x)$ is $\lambda$-strongly-convex if for any $x,y\in \mathcal{K}$, the following holds:
$$
f(y)\ge f(x)+\nabla f(x)^{\top}(y-x)+\frac{\lambda}{2}\|x-y\|_2^2
$$
\end{definition}

\begin{definition}
A function $f(x)$ is $\alpha$-exp-concave if $e^{-\alpha f(x)}$ is a convex function. 
\end{definition}

\section{A More Efficient Adaptive Regret Algorithm}\label{s3}

\begin{algorithm}[t]
\caption{Efficient Follow-the-Leading-History (EFLH) - Basic Version}
\label{alg1}
\begin{algorithmic}[1]
\STATE Input: OCO algorithm $\mA$, active expert set $S_t$.
\STATE Let $\mA_t$ be an instance of $\mA$ in initialized at time $t$. Initialize the set of active experts: $S_1=\{1\}$, with initial weight $w_1^{(1)}=\frac{1}{2GD}$.
\STATE Pruning rule: for $k\ge 1$, the lifespan $l_t$ of $\A_t$ with integer $t=r 2^{2^k-1}$ is $2^{2^k+1}$ ($=4$ if $2\nmid t$), where $2^{2^k+1}\nmid t$. "Deceased" experts will be removed from the active expert set $S_t$.
\FOR{$t = 1, \ldots, T$}
\STATE Let $W_t=\sum_{j\in S_t} w_t^{(j)}$.
\STATE Play $x_t=\sum_{j\in S_t} \frac{w_t^{(j)}}{W_t} x_t^{(j)}$, where $x_t^{(j)}$ is the prediction of $\mA_j$.
\FOR { $j\in S_t$ }
\STATE $$
w_{t+1}^{(j)}=w_t^{(j)} \left(1+\frac{1}{GD} \min\left\{\frac{1}{2},\sqrt{\frac{\log T}{l_j}}\right\} (\ell_t(x_t)-\ell_t(x_t^{(j)})) \right)
$$
\ENDFOR
\STATE Update $S_t$ according to the pruning rule and add $t + 1$ to get $S_{t+1}$. Initialize 
$$
w_{t+1}^{t+1}=\frac{1}{GD} \min\left\{\frac{1}{2},\sqrt{\frac{\log T}{l_{t+1}}}\right\}
$$
\ENDFOR
\end{algorithmic}
\end{algorithm}

The Follow-the-Leading History (FLH) algorithm \cite{hazan2009efficient}  achieved $\tilde{O}(\sqrt{T})$ adaptive regret by initiating different OCO algorithms at each time step, then treating them as experts and running a multiplicative weight method to choose between them. This can be extended to attain $\tilde{O}(\sqrt{|I|})$ adaptive regret bound by using parameter-free OCO algorithms as experts, or by setting the $\eta$ in the multiplicative weight algorithm as in \cite{daniely2015strongly}.

Although these algorithms can achieve a near-optimal $\tilde{O}(\sqrt{|I|})$ adaptive regret, they have to use $\Theta(\log T)$  experts per round. We propose a more efficient algorithm \ref{alg1} which achieves vanishing regret and uses only $O(\log \log T)$ experts. 

The intuition for our algorithm stems from the FLH method, in which the experts' lifespan is of form $2^k$. We denote lifespan as the length of the interval that this expert is run on (it also chooses its parameters optimally according to its lifespan). This leads to $\Theta(\log T)$ number of active experts per round, and we could potentially improve it to $O(\log \log T)$ if we change the lifespan to be $2^{2^k}$. While this does increase the regret, we achieve an $\tilde{O}(|I|^{\frac{3}{4}})$  regret bound. The formal regret guarantee is given below.

\begin{theorem}\label{main}
When using OGD as the expert algorithm $\mA$, Algorithm \ref{alg1} achieves the following adaptive regret bound over any interval $I\subset [1,T]$ for genral convex loss
$$ \text{SA-Regret}(\mA) = 36 GD \sqrt{\log T} \cdot |I|^{\frac{3}{4}}$$
with $O(\log \log T)$ experts.
\end{theorem}

We use the Online Gradient Descent \cite{zinkevich2003online} algorithm in the theorem for its $\frac{3}{2}GD\sqrt{T}$ regret bound, see \cite{hazan2016introduction}.The sub-optimal $\tilde{O}(|I|^{\frac{3}{4}})$ rate can be improved to be closer to the optimal $\tilde{O}(|I|^{\frac{1}{2}})$ rate while still using only $O(\log \log T)$ number of experts. We will discuss this improvement in the next section.

\subsection{Proof of Theorem \ref{main}}
    We use OGD as the base algorithm $\A$.
    Without loss of generality we only need to consider intervals with length at least 8. The proof idea is to derive a recursion of regret bounds, and use induction on the interval length. The key observation is that, due to the double-exponential construction of interval lengths, for any interval $[s,t]$, it's guaranteed that a sub-interval in the end with length at least $\sqrt{t-s}/2$ is covered by some expert. In the meantime, the number of 'active' experts per round is at most $O(\log \log T)$. We formalize the above observation in the two following lemmas.
    \begin{lemma}\label{p1}
        For any interval $I=[s,t]$, there exists an integer $i\in [s,t-\sqrt{t-s}/2]$, such that $\mA_i$ is alive throughout $[i,t]$.
    \end{lemma}
    \begin{proof}
        Assume $2^{2^k}\le t-s\le 2^{2^{k+1}}$, then $\sqrt{t-s}/2\le 2^{2^k-1}$. Notice that $t\ge 2^{2^k}+1$. Assume $r\ge 2$ is the largest integer such that $r 2^{2^k-1}\le t$, then one of $i=(r-1) 2^{2^k-1}$ and $i=(r-2) 2^{2^k-1}$ is satisfactory because its lifespan is $ 2^{2^k+1}\ge 3\times 2^{2^k-1}$. The reason we consider two candidate $i$ is that when $r\ge 2$, one of $r-2$ and $r-1$ is odd and we use that to guarantee the lifespan isn't strictly larger than $ 2^{2^k+1}$ (such choice also excludes the potential bad case $r-2=0$). 
    \end{proof}
    
    In fact, Lemma \ref{p1} implies an even stronger argument for the coverage of $ [t-\sqrt{t-s}/2,t]$, that is $2^{2^k-1}\le t-i\le 2^{2^k+1}$, and as a result $\eta=\frac{1}{\sqrt{2^{2^k+1}}}$ is optimal (up-to-constant) for this chosen expert. This property means that we don't need to tune $\eta$ optimally for the length $\sqrt{t-s}/2$, but only need to tune $\eta$ with respect to the lifespan of the expert itself. For example, the OGD algorithm $\A_i$ achieves (nearly) optimal regret on $[i,t]$ as well because the optimal learning rate for $[i,t]$ is the same as that for $[i,i+l_i]$ up to a constant factor of 2. To see this, notice that $l_i\ge t-i$ and $t-i\ge \frac{l_i}{4}$.
    
    \begin{lemma}\label{p2}
        $|S_t|=O(\log \log T)$.
    \end{lemma}
    \begin{proof}
    At any time up to $T$, there can only be $O(\log \log T)$ different lifespans sizes by the algorithm definition.
        Notice that  for any  $k$, though the total number of experts with lifespan of $2^{2^k+1}$ might be large, the number of active experts with lifespan of $2^{2^k+1}$ is only at most 4 which concludes the proof. 
    \end{proof}
    
    Lemma \ref{p2} already proves the efficiency claim of Theorem \ref{main}. To bound the regret we make an induction on the length of interval $|I|$. Let $2^{2^k}\le |I|\le 2^{2^{k+1}}$, we will prove by induction on $|I|$. We need the following technical lemma on the recursion of regret.
    \begin{lemma}\label{tech}
    For any $x\ge 1$, we have that
    $$6 x^{\frac{3}{4}}\ge 6(x-x^{\frac{1}{2}}/2)^{\frac{3}{4}}+(x^{\frac{1}{2}}/2)^{\frac{1}{2}} $$
    \end{lemma}
    \begin{proof}
        Let $y=(x^{\frac{1}{2}}/2)^{\frac{1}{2}}$, after simplification the above inequality becomes 
        \begin{align*}
            &6 x^{\frac{3}{4}}\ge 6(x-x^{\frac{1}{2}}/2)^{\frac{3}{4}}+(x^{\frac{1}{2}}/2)^{\frac{1}{2}}\\
            \iff &12\sqrt{2}y^3\ge 6(4y^4-y^2)^{\frac{3}{4}} +y\\
            \iff &(12\sqrt{2}y^3-y)^4\ge 6^4 (4y^4-y^2)^3\\
            \iff & (12\sqrt{2}y^2-1)^4 \ge 1296y^2(4y^2-1)^3\\
            \iff & (62208-13824\sqrt{2})y^6 -13824y^4+(1296-48\sqrt{2})y^2+1\ge 0
        \end{align*}
        The derivative of the LHS is non-negative because $y\ge 1$ and $62208-13824\sqrt{2}\ge 13824$. This proves the LHS is monotonely increasing in $y$, and we only need to prove its non-negativity when $y=1$, which can be verified by straight calculation. 
    \end{proof}
    
    The first step is to derive a regret bound on the sub-interval $[i,t]$ which is covered by a single expert $\A_i$. The regret on $[i,t]$ can be decomposed as the sum of the expert regret and the multiplicative weight regret to choose that best expert in the interval. The expert regret is upper bounded by $3 GD \sqrt{t-i}$ due to the optimality of $\A_i$ while the multiplicative weight regret can be upper bounded by $3GD \sqrt{\log T(t-i)}$ as shown in the following lemma, the proof is left to the appendix.
    
    \begin{lemma}\label{mw regret}
        For the $i$ and $\A_i$ chosen in Lemma \ref{p1}, the regret of Algorithm \ref{alg1} over the sub-interval $[i,t]$ is upper bounded by $3GD \sqrt{t-i}+3GD\sqrt{\log T(t-i)}$.
    \end{lemma}
    
    Now we have gathered all the pieces we need to prove our induction.
    \paragraph{Base case:} for $|I|=1$, the regret is upper bounded by $GD \le 36 GD \sqrt{\log T} \cdot 1^{\frac{3}{4}}$. 
    \paragraph{Induction step:} suppose for any $|I|<m$ we have the regret bound in the statement of theorem. Consider now $t-s=m$, from Lemma \ref{p1} we know there exists an integer $i\in [s,t-\sqrt{t-s}/2]$, such that $\mA_i$ is alive throughout $[i,t]$. Algorithm \ref{alg1} guarantees an 
    $$3GD \sqrt{t-i}+3 GD\sqrt{\log T(t-i)}\le  6GD \sqrt{\log T} (t-i)^{\frac{1}{2}}$$
    regret over $[i,t]$ by Lemma \ref{mw regret}, and by induction the regret over $[s,i]$ is upper bounded by\newline $36 GD \sqrt{\log T} (i-s)^{\frac{3}{4}}$. By the monotonicity of the function $f(y)=6(x-y)^{\frac{3}{4}}+\sqrt{y}$ when the variable $y\ge \sqrt{x}/2$, we reach the desired conclusion by using Lemma \ref{tech}:
    $$
    6(t-i)^{\frac{1}{2}}+36 (i-s)^{\frac{3}{4}}\le 6(\frac{\sqrt{t-s}}{2})^{\frac{1}{2}}+36 (t-s-\frac{\sqrt{t-s}}{2})^{\frac{3}{4}}\le 36 (t-s)^{\frac{3}{4}}
    $$
    To see the monotonicity, we use the fact $y\ge \sqrt{x}/2$ to see that
    $$
    f'(y)=\frac{1}{2\sqrt{y}}-\frac{9}{2 (x-y)^{\frac{1}{4}}}
        \le \frac{1}{2\sqrt{y}}-\frac{9}{2 x^{\frac{1}{4}}}
        \le \frac{1}{2\sqrt{y}}-\frac{9}{2 \sqrt{2y}}\le 0
    $$

\section{Approaching the Optimal Rate}\label{s4}
\begin{algorithm}[t]
\caption{Efficient Follow-the-Leading-History (EFLH) - Full Version}
\label{alg2}
\begin{algorithmic}[1]
\STATE Input: OCO algorithm $\mA$, active expert set $S_t$, horizon $T$ and constant $\epsilon>0$.
\STATE Pruning rule: let $\mA_{(t,k)}$ be an instance of $\mA$ initialized at $t$ with lifespan $4 l_k=4\lfloor 2^{(1+\epsilon)^k}/2 \rfloor+4$, for $2^{(1+\epsilon)^k}/2\le T$. "Deceased" experts will be removed from the active expert set $S_t$.

\STATE Initialize: $S_1=\{(1,1),(1,2),...\}$, $w_1^{(1,k)}=\frac{1}{GD} \min \left\{\frac{1}{2}, \sqrt{\frac{\log T}{l_k}}\right\}$.
\FOR{$t = 1, \ldots, T$}
\STATE Let $W_t=\sum_{(j,k)\in S_t} w_t^{(j,k)}$.
\STATE Play $x_t=\sum_{(j,k)\in S_t} \frac{w_t^{(j,k)}}{W_t} x_t^{(j,k)}$, where $x_t^{(j,k)}$ is the prediction of $\mA_{(j,k)}$.
\STATE Perform multiplicative weight update to get $w_{t+1}$. For $(j,k)\in S_t$
$$
w_{t+1}^{(j,k)}=w_t^{(j,k)} \left(1+\frac{1}{GD} \min \left\{\frac{1}{2},\sqrt{\frac{\log T}{l_k}}\right\} (\ell_t(x_t)-\ell_t(x_t^{(j,k)})) \right)
$$
\STATE Update $S_t$ according to the pruning rule. Initialize 
$$
w_{t+1}^{(t+1,k)}=\frac{1}{GD} \min \left\{\frac{1}{2}, \sqrt{\frac{\log T}{l_k}}\right\}
$$
if $(t+1,k)$ is added to $S_{t+1}$ (when $l_k | t$).
\ENDFOR
\end{algorithmic}
\end{algorithm}

The basic approach given in the previous section achieves vanishing adaptive regret with only $O(\log \log T)$ number of experts, improving the efficiency of previous works \cite{hazan2009efficient,daniely2015strongly}.  In this section, we extend the basic version of Algorithm \ref{alg1} and show how to achieve an $\tilde{O}(|I|^{\frac{1+\epsilon}{2}})$ adaptive regret bound with $O( \log \log T/\eps)$ number of experts. 


The intuition stems from the recursion of regret bounds. Suppose the construction of our experts guarantees that for any interval with length $x$, there exists a sub-interval with length $\Theta(x^{\alpha})$ in the end which is covered by some expert with the same initial time, for some constant $\alpha \ge 0$. Then similarly, we need to solve the recursion of a regret bound function $g$ such that
$$
g(x)\ge g(x-x^{\alpha})+x^{\frac{\alpha}{2}} ,
$$
which approximately gives the solution of $g(x)=\Theta(x^{1-\frac{\alpha}{2}})$. To approach the optimal rate we set $\alpha=1-\epsilon$, giving an $\tilde{O}(|I|^{\frac{1+\epsilon}{2}})$ regret bound. It remains to describe an explicit construction that guarantees a covering with $\alpha=1-\epsilon$.

Suppose our construction contains experts with lifespan of the form $f(n)$, then it's equivalent to require that $f(n+1)^{1-\epsilon}\sim f(n)$ which is approximately $f(n+1)\sim f(n)^{1+\epsilon}$. Initializing $f(1)=2$, for example, gives an alternative choice of double-exponential lifespan $2^{(1+\epsilon)^k}$. 

We also need to slightly modify how we define the experts and the pruning rule, since $2^{(1+\epsilon)^k}$ isn't necessarily an integer now. Define $l_k=\lfloor 2^{(1+\epsilon)^k}/2 \rfloor+1$, we hold experts with lifespan $4 l_k$ for every $k$ satisfying $2^{(1+\epsilon)^k}/2\le T$. Additionally, we initialize an expert with lifespan $4 l_k$ at time $t$ if $l_k \mid (t-1)$, notice that this might create multiple experts with the same initial time point in contrast to Algorithm \ref{alg1}. The resulting Algorithm \ref{alg2} has the following regret guarantee.

\begin{theorem}\label{main2}
When using OGD as the expert algorithm $\mA$, Algorithm \ref{alg2} achieves the following adaptive regret bound over any interval $I\subset [1,T]$ for genral convex loss
$$ \text{SA-Regret}(\mA) = 48cGD \sqrt{\log T} |I|^{\frac{1+\epsilon}{2}}$$
with $O(\log \log T/\epsilon)$ experts.
\end{theorem}
The proof is essentially the same as that of Theorem \ref{main} which we leave to appendix, the main new step is to derive a generalized version of Lemma \ref{tech}, which roughly says that
$$
x^{\frac{1-\epsilon}{2}}=O(x^{\frac{1+\epsilon}{2}}-(x-x^{1-\epsilon})^{\frac{1+\epsilon}{2}})
$$

\section{Limits of the History Lookback Technique}\label{s6}

\begin{figure}[ht]
    \centering
    \includegraphics[scale=0.6]{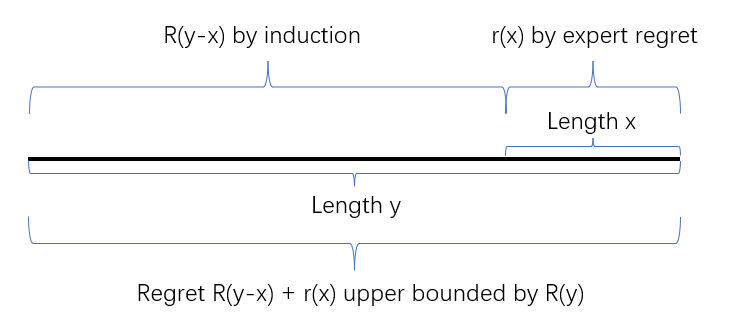}
    \caption{Illustration of the history lookback technique}
    \label{fig:log}
\end{figure}

In this section we discuss the limitation of the history lookback technique, which is used to derive all our results. The basic idea of the history lookback technique is to use recursion to bound the adaptive regret: for any interval of length $y$, there is guaranteed to be an expert initiated in the end covering a smaller interval with length $x(y)$. The expert guarantees some regret $r(x)$ over the small interval, and we denote the regret over the rest of the large interval as $R(y-x)$.

Now $R(y-x)+r(x)$ becomes a regret bound over the large interval, and we would like to find $R$ satisfying $R(y-x)+r(x)\le R(y)$, allowing us to use induction on the length of interval to get an adaptive regret bound $R(I)$ for any interval with length $I$. The result of \cite{hazan2009efficient} for general convex loss, for example, can be interpreted as a special case of setting $r(x)=\log T \sqrt{x}$, $x(y)=\frac{y}{4}$ and $R(x)=5\log T \sqrt{x}$.

Typically the function $r(x)$ is determined by the problem itself. Still, the interval length evolution $x(y)$ is adjustable, and we aim to find the smallest $x(y)$ (which means most efficient) that maintains a near-optimal $R(x)$. We would like to find tight trade-off between $R$ and $x$ in the following inequality:
$$
R(y-x(y))+r(x(y))\le R(y)
$$
In particular, we are interested in how small $x(y)$ can be when $r(x)=\sqrt{x}$, while maintaining $R(x)=o(x)$. We have the following impossibility result. 
\begin{proposition}\label{prop:lb}
Suppose $r(x)=C_1 \sqrt{x}$, when $0<x(y)<\min\{C_2 y^{\frac{1}{n}}, \frac{y}{2}\}$ for $y>1$ with some constants $C_1>0, C_2 \ge 1$, then any $R(\cdot)$ satisfying
$$
R(y-x(y))+r(x(y))\le R(y)
$$
must be lower bounded by $R(y)\ge \frac{C_1}{2\sqrt{C_2}} y^{1-\frac{1}{2n}}$.
\end{proposition}

Proposition \ref{prop:lb} indicates that we cannot maintain computation better than double-log and vanishing regret at the same time. For example, if we take interval length to be of form $2^{2^{2^k}}$ (corresponding to $y=2^{2^{2^{k+1}}}$ and $x=\Omega(2^{2^{2^k}})$) such that the number of experts is just $O(\log \log \log T)$, it leads to an undesirable regret bound $R(I)=\O(I^{1-\frac{1}{2\log T}})$. 

Using the same approach, we can also find the optimal $x(y)$ when the expert regret bound $r(x)=x^{\alpha}$ and the desired adaptive regret bound $R(x)=x^{\beta}$ are given.

\begin{proposition} \label{cor:lb}
Let $r(x)=x^{\alpha}, R(x)=x^{\beta}$ where $0\le \alpha\le \frac{1}{2}, \alpha < \beta <1$. 
Then choosing $x(y)=\Theta(y^{\frac{1-\beta}{1-\alpha}})$ guarantees the following:
\begin{enumerate}
    \item $ r(x(y))\le \frac{2}{\beta}[R(y)-R(y-x(y))]$ is satisfied, thus the adaptive regret is bounded by $ O( \frac{x^\beta}{\beta})$. 
    
    \item
The computational overhead is  $O(\frac{1-\beta}{\beta-\alpha}\log \log T)$.
\end{enumerate}
Meanwhile, any choice of $x(y)=o(y^{\frac{1-\beta}{1-\alpha}})$ will violate the first property.
\end{proposition}

The proof to Proposition \ref{prop:lb} and Proposition \ref{cor:lb} are deferred it to the appendix. As an implication, the case of $\alpha=0$ will be used in the next section to derive more efficient dynamic regret algorithms.
\section{Efficient Dynamic Regret Minimization} \label{s5}
\begin{algorithm}[h!]
\caption{Efficient Follow-the-Leading-History (EFLH) - Exp-concave Version}
\label{alg3}
\begin{algorithmic}[1]
\STATE Input: OCO algorithm $\mA$, active expert set $S_t$, horizon $T$, exp-concave parameter $\alpha$ and $\epsilon>0$.
\STATE Pruning rule: let $\mA_{(t,k)}$ be an instance of $\mA$ initialized at $t$ with lifespan $4 l_k=4\lfloor 2^{(1+\epsilon)^k}/2 \rfloor+4$, for  only the largest $l_k|t-1$ satisfying $2^{(1+\epsilon)^k}/2\le T$. 

\STATE Initialize: $S_1=\{(1,1),(1,2),...\}$, $w_1^{(1,k)}=\frac{1}{|S_1|}$.
\FOR{$t = 1, \ldots, T$}
\STATE Play $x_t=\sum_{(j,k)\in S_t} w_t^{(j,k)} x_t^{(j,k)}$, where $x_t^{(j,k)}$ is the prediction of $\mA_{(j,k)}
$
\STATE Perform multiplicative weight update to get $w_{t+1}$. For $(j,k)\in S_t$
$$
\hat{w}_{t+1}^{(j,k)}=\frac{w_t^{(j,k)} e^{-\alpha \ell_t(x_t^{(j,k)})}}{\sum_{(i,k)\in S_t}w_t^{(i,k)} e^{-\alpha \ell_t(x_t^{(i,k)})}}
$$
\STATE Update $S_t$ according to the pruning rule. Set and update for all $j\le t$
$$w_{t+1}^{(t+1,k)}=\frac{1}{t+1} \ , \ w_{t+1}^{(j,k)}=(1-\frac{1}{t+1})\hat{w}_{t+1}^{(j,k)}, $$
where $w_{t+1}^{(t+1,k)}$ is the weight of the newly added expert $\mA_{(t+1,k)}$.
\ENDFOR
\end{algorithmic}
\end{algorithm}

In this section we show how to achieve near optimal dynamic regret with a more efficient algorithm as compared to  state of the art.  When the loss functions are exp-concave or strongly-convex, running Algorithm \ref{alg3} with experts being Online Newton Step (ONS) \cite{hazan2007logarithmic} or Online Gradient Decent (OGS) respectively gives near-optimal dynamic regret bound. 

Algorithm \ref{alg3} is a simplified version of Algorithm \ref{alg2}. The main difference is that Algorithm \ref{alg3} does not require learning rate tuning, since we no longer need interval length dependent regret bounds as in the general convex case. 

\begin{theorem}
Algorithm \ref{alg3} achieves the following dynamic regret bound for exp-concave (with $\mA$ being ONS) or strongly convex (with $\mA$ being OGD) loss functions
$$ \text{D-Regret}(\mA) = \sum_t \ell_t(x_t) - \min_{x_{1:T}^*} \ell_t(x_t^*)= \tilde{O}(\frac{T^{\frac{1}{3}+\epsilon} \mathcal{P}^{\frac{2}{3}}}{\epsilon}) ,$$
where $\mathcal{P}=\sum_{t=1}^T \|x_{t+1}^*-x_t^*\|_1$. Further, the number of active experts is  $O(\frac{\log \log T}{\epsilon})$. 
\end{theorem}

\begin{proof}
The proof follows by observing that both Theorem 14 in \cite{baby2021optimal} and Theorem 8 in \cite{baby2022optimal} only make use of FLH as an adaptive regret black-box. We maintain the low-level experts: ONS for exp-concave loss and OGD for strongly-convex loss, but replace FLH by Algorithm \ref{alg3}.

To proceed, we first show that Algorithm \ref{alg3} can be applied to exp-concave or strongly-convex loss functions, but at the cost of a worse adaptive regret bound compared with the $O(\log^2 T)$ bound of FLH.

\begin{lemma}\label{lem exp}
Assume  $\mA$ guarantees a regret bound of $O(\log T)$. Algorithm \ref{alg3} achieves the following adaptive regret bound over any interval $I\subset [1,T]$ for exp-concave or strongly convex loss
$$ \text{SA-Regret}(\mA) = O(\frac{ I^\epsilon \log T }{\epsilon})$$
with $O(\frac{\log \log T}{\epsilon})$ experts. 
\end{lemma}

The proof of Lemma \ref{lem exp} is identical to that of Theorem \ref{main}, except that the regret of experts and the recursion are different. The regret of experts are guaranteed to be $O(\log I)$ by using ONS \cite{hazan2007logarithmic} as the expert algorithm $\A$ for exp-concave loss, or by using OGD for strongly-convex loss. We only need to solve the recursion when interval length of form $2^{(1+\epsilon)^k}$ is used.

According to Lemma \ref{lem exp}, Algorithm \ref{alg3} achieves a worse regret $O(\frac{ I^{\epsilon}\log T}{\epsilon})$ instead of $O(\log^2 T)$ of FLH. Fortunately, the regret bounds of \cite{baby2021optimal}, \cite{baby2022optimal} are achieved by summing up the regret of FLH over $O(T^{\frac{1}{3}}\mathcal{P}^{\frac{2}{3}})$ number of intervals, therefore by using Algorithm \ref{alg3} instead of FLH we get a final bound $\tilde{O}(\frac{ T^{\frac{1}{3}+\epsilon}\mathcal{P}^{\frac{2}{3}}}{\epsilon})$. To this end, we extract the following proposition, from their result.

\begin{proposition}[Lemma 30 + Lemma 31 + Theorem 14 in \cite{baby2021optimal}]
There exists a partition $P=\cup_{i=1}^M I_i$ of the whole interval with size $M=O(T^{\frac{1}{3}}\mathcal{P}^{\frac{2}{3}})$, such that the over all dynamic regret is bounded by
$$
\text{D-Regret}\le \sum_{i=1}^M (\text{Regret}_{\mA}(I_i)+\text{Regret}_{\mA_{\text{meta}}}(I_i)+\tilde{O}(1))
$$
where $\text{Regret}_{\mA}(I_i)$ is the regret of the best expert (ONS/OGD) over interval $I_i$, and $\text{Regret}_{\mA_{\text{meta}}}(I_i)$ is the regret of the meta algorithm (FLH/Algorithm \ref{alg3}) over the best expert $\mA$ over interval $I_i$.
\end{proposition}

Putting $M=O(T^{\frac{1}{3}}\mathcal{P}^{\frac{2}{3}})$ and $\text{Regret}_{\mA_{\text{meta}}}(I)=O(\frac{I^{\epsilon}\log T}{\epsilon})\le O(\frac{ T^{\epsilon}\log T}{\epsilon})$ together we get the desired regret guarantee.

The overall computation consists of the number of experts in Algorithm \ref{alg3}, and the computation of each expert. For exp-concave loss we use ONS as the expert which has $O(d^2)$ computation, thus the overall computation is $O(\frac{d^2 \log \log T}{\epsilon})$. While for strongly-convex loss, OGD is used as the expert, and the overall computation is $O(\frac{d \log \log T}{\epsilon})$.

\end{proof}

\section{Conclusion}\label{s7}
In this paper we propose a more efficient reduction from regret minimization algorithms to adaptive and dynamic regret minimization.  We apply a new construction of experts with doubly exponential lifespans $2^{(1+\epsilon)^k}$, then obtain an $\tilde{O}(|I|^{\frac{1+\epsilon}{2}})$ adaptive regret bound with $O( \log \log T/\eps)$ number of experts. As an implication, we show that $O( \log \log T/\eps)$ number of experts also suffices for near-optimal dynamic regret. Our result characterizes the trade-off between regret and efficiency in minimizing adaptive regret in online learning, showing how to achieve near-optimal adaptive regret bounds with $O(\log \log T)$ number of experts. 

We have also shown that the technique of history look-back cannot be used to further improve the number of experts in a reduction from regret to adaptive regret, if the regret is to be near optimal.  Can we go beyond this technique to improve computational efficiency even further? 

\section*{Acknowledgement}
We thank Ohad Shamir, Qinghua Liu, Yuanyu Wan and Xinyi Chen for helpful comments and suggestions. 

\bibliography{Xbib}
\bibliographystyle{plain}

\newpage
\appendix
\section{Proof of Lemma \ref{mw regret}}
\begin{proof}
The expert regret is upper bounded by $3GD\sqrt{t-i}$ due to the optimality of $\A_i$, and the choice of $\eta$ is optimal up to a constant factor of 2. We only need to upper bound the regret of the multiplicative weight algorithm. We focus on the case that $\sqrt{\frac{\log T}{l_j}} \le \frac{1}{2}$, because in the other case the length $t-i$ of the sub-interval is $O(\log T)$, and its regret is upper bounded by $(t-i)GD =O(GD\sqrt{\log T (t-i)})$, and the conclusion follows directly.
    
We define the pseudo weight $\tilde{w}_t^{(j)}=GD\sqrt{\frac{l_j}{\log T}} w_t^{(j)}$ for $i\le t\le i+l_i$, and for $t>i+l_i$ we just set $\tilde{w}_t^{(j)}=\tilde{w}_{i+l_i}^{(j)}$. Let $\tilde{W}_t=\sum_{j\in S_t} \tilde{w}_t^{(j)}$, we are going to show the following inequality
\begin{equation}\label{eq3}
    \tilde{W}_t\le t 
\end{equation}
We prove this by induction. For $t=1$ it follows from the fact that $\tilde{W}_1=1$. Now we assume it holds for all $t'\le t$. We have
\begin{align*}
    \tilde{W}_{t+1}&=\sum_{j\in S_{t+1}} \tilde{w}_{t+1}^{(j)}\\
    &=\tilde{w}_{t+1}^{(t+1)}+\sum_{j\in S_{t+1}, j\le t} \tilde{w}_{t+1}^{(j)}\\
    &\le 1+ \sum_{j\in S_{t+1}, j\le t} \tilde{w}_{t+1}^{(j)}\\
    &= 1+ \sum_{j\in S_{t+1}, j\le t} \tilde{w}_t^{(j)} \left(1+\frac{1}{GD} \sqrt{\frac{\log T}{l_j}} (\ell_t(x_t)-\ell_t(x_t^{(j)})) \right)\\
    &=1+ \tilde{W}_t+ \sum_{j\in S_t} \tilde{w}_t^{(j)} \frac{1}{GD} \sqrt{\frac{\log T}{l_j}} (\ell_t(x_t)-\ell_t(x_t^{(j)}))\\
    &=1+ \tilde{W}_t+ \sum_{j\in S_t} w_t^{(j)}  (\ell_t(x_t)-\ell_t(x_t^{(j)}))\\
    &\le t+1 + \sum_{j\in S_t} w_t^{(j)}  (\ell_t(x_t)-\ell_t(x_t^{(j)}))
\end{align*}

We further show that $\sum_{j\in S_t} w_t^{(j)}  (\ell_t(x_t)-\ell_t(x_t^{(j)}))\le  0$:
\begin{align*}
    \sum_{j\in S_t} w_t^{(j)}  (\ell_t(x_t)-\ell_t(x_t^{(j)}))&=W_t \sum_{j\in S_t} \frac{w_t^{(j)}}{W_t}  (\ell_t(x_t)-\ell_t(x_t^{(j)}))\\
    &= W_t \sum_{j\in S_t}(\ell_t(x_t)-\ell_t(x_t))\\
    &=0
\end{align*}
which finishes the proof of induction.

By inequality \ref{eq3}, we have that
$$
    \tilde{w}_{t+1}^{(i)}\le \tilde{W}_{t+1} \le t+1
$$
Taking the logarithm of both sides, we have
$$
    \log(\tilde{w}_{t+1}^{(i)})\le \log(t+1)
$$
Recall the expression
$$
    \tilde{w}_{t+1}^{(i)}=\prod_{\tau=i}^t \left(1+\frac{1}{GD} \sqrt{\frac{\log T}{l_i}} (\ell_{\tau}(x_{\tau})-\ell_{\tau}(x_{\tau}^{(i)}))\right)
$$
By using the fact that $\log(1+x)\ge x-x^2, \forall x\ge -1/2$ and 
$$
   |\frac{1}{GD} \sqrt{\frac{\log T}{l_i}} (\ell_{\tau}(x_{\tau})-\ell_{\tau}(x_{\tau}^{(i)}))|\le 1/2 
$$
we obtain
\begin{align*}
    \log(\tilde{w}_{t+1}^{(i)})&\ge \sum_{\tau=i}^t \frac{1}{GD} \sqrt{\frac{\log T}{l_i}} (\ell_{\tau}(x_{\tau})-\ell_{\tau}(x_{\tau}^{(i)}))-\sum_{\tau=i}^t [\frac{1}{GD} \sqrt{\frac{\log T}{l_i}} (\ell_{\tau}(x_{\tau})-\ell_{\tau}(x_{\tau}^{(i)}))]^2\\
    &\ge \sum_{\tau=i}^t \frac{1}{GD} \sqrt{\frac{\log T}{l_i}} (\ell_{\tau}(x_{\tau})-\ell_{\tau}(x_{\tau}^{(i)}))-\frac{\log T}{l_i} (t-i)
\end{align*}

Combing this with $\log(\tilde{w}_{t+1}^{(i)})\le \log(t+1)$, we have that 
\begin{align*}
    \sum_{\tau=i}^t (\ell_{\tau}(x_{\tau})-\ell_{\tau}(x_{\tau}^{(i)}))&\le
    \frac{1}{GD} \sqrt{\frac{\log T}{l_i}} (t-i)+\frac{1}{GD} \sqrt{\frac{l_i}{\log T}} \log(t+1)
\end{align*}
Notice that $\frac{1}{4}l_i\le t-i\le  l_i$ from Lemma \ref{p1} and $1+2=3$, we conclude the proof.

\end{proof}

\section{Proof of Theorem \ref{main2}}
The proof is essentially the same as that of Theorem \ref{main}, the main new step is to derive a generalized version of Lemma \ref{tech}.

\begin{lemma}\label{tech new}
    For any $x\ge 1$, for $\epsilon < \frac{1}{2}$, we have that
    $$8 x^{\frac{1+\epsilon}{2}}\ge 8(x-x^{1-\epsilon}/2)^{\frac{1+\epsilon}{2}}+(x/2)^{\frac{1-\epsilon}{2}} $$
\end{lemma}
\begin{proof}
    We would like to upper bound the term $(x-x^{1-\epsilon}/2)^{\frac{1+\epsilon}{2}}$. Notice that $0<x^{-\epsilon}<1$, we have that 
    $$
    (1-x^{-\epsilon}/2)^{\frac{1+\epsilon}{2}}=e^{\frac{1+\epsilon}{2} \log (1-x^{-\epsilon}/2) }\le e^{-\frac{1+\epsilon}{4} x^{-\epsilon}}\le 1-\frac{1+\epsilon}{8} x^{-\epsilon}
    $$
    where the last step follows from $e^{-x}\le 1-\frac{x}{2}$ when $0<x\le 1$. The above estimation gives us $x^{\frac{1+\epsilon}{2}}-(x-x^{1-\epsilon}/2)^{\frac{1+\epsilon}{2}}\ge \frac{1+\epsilon}{8}x^{\frac{1-\epsilon}{2}}$ which concludes our proof.
\end{proof}

We go through the rest of the proof, and omit details which are the same as Theorem \ref{main}. The number of active experts per round is upper bounded by $4 \log_{1+\epsilon}\log_2  T=O(\log \log T/\epsilon)$, since at each time step there are at most 4 active experts with lifespan $4l_k$ for any $k$. 

As for the regret bound, similarly we have the following property on the covering of intervals.
\begin{lemma}\label{p1 new}
        For any interval $I=[s,t]$, there exists an integer $i\in [s,t-(t-s)^{1-\epsilon}/2]$, such that $\mA_i$ is alive throughout $[i,t]$.
\end{lemma}
And the choice of $\eta=\sqrt{\frac{1}{l_k}}$ is still optimal for each expert up to a constant factor of 2. An almost identical analysis of Lemma \ref{mw regret} yields the following (the only difference is that we make induction on $\tilde{W}_t\le \frac{4\log \log T}{\epsilon} t$ instead, which doesn't affect the bound because $\log(\frac{\log \log T}{\epsilon})=o(\log T)$). 

\begin{lemma}\label{mw regret new}
        For the $i$ and $\A_{(i,j)}$ chosen in Lemma \ref{p1 new}, the regret of Algorithm \ref{alg2} over the sub-interval $[i,t]$ is upper bounded by $3GD \sqrt{t-i}+3GD \sqrt{\log T(t-i)}$.
\end{lemma}
The reason of such difference is that at time $t=1$ there are multiple active experts in Algorithm \ref{alg2} while there is just one in Algorithm \ref{alg1}. It's possible to make the proof simpler as that of Lemma \ref{mw regret}, however it would complicate the algorithm itself. We proceed to state our induction on $|I|$.
    \paragraph{Base case:} for $|I|=1$, the regret is upper bounded by $GD\le 48GD \sqrt{\log T} \cdot 1^{\frac{1+\epsilon}{2}}$. 
    \paragraph{Induction step:} suppose for any $|I|<m$ we have the regret bound in the statement of theorem. Consider now $t-s=m$, from Lemma \ref{p1 new} we know there exists an integer $i\in [s,t-(t-s)^{1-\epsilon}/2]$ and $k$ satisfying $l_k\le (t-s)^{1-\epsilon}/2\le 4l_k$, such that $\mA_{(i,k)}$ is alive throughout $[i,t]$. Algorithm \ref{alg2} guarantees an 
    $$3GD  \sqrt{t-i}+3 GD\sqrt{\log T(t-i)}\le  6GD \sqrt{\log T} (t-i)^{\frac{1}{2}}$$
    regret over $[i,t]$ by Lemma \ref{mw regret new}, and by induction the regret over $[s,i]$ is upper bounded by \newline $48GD \sqrt{\log T} (i-s)^{\frac{1+\epsilon}{2}}$. By the monotonicity of the function $f(y)=8(x-y)^{\frac{1+\epsilon}{2}}+\sqrt{y}$ when the variable $y\ge x^{1-\epsilon}/2$, we reach the desired conclusion by Lemma \ref{tech new}. To see the monotonicity, we use the fact $y\ge x^{1-\epsilon}/2$ to see that
    \begin{align*}
        f'(y)&=\frac{1}{2\sqrt{y}}-\frac{4(1+\epsilon)}{ (x-y)^{\frac{1-\epsilon}{2}}}\\
        &\le \frac{1}{2\sqrt{y}}-\frac{4(1+\epsilon)}{ x^{\frac{1-\epsilon}{2}}}\\
        &\le \frac{1}{2\sqrt{y}}-\frac{4(1+\epsilon)}{\sqrt{2y}}\\
        &\le 0
    \end{align*}

\section{Proof of Proposition \ref{prop:lb}}

We prove by induction. For $y=1$, it follows that $R(1)\ge r(1)=C_1$. Suppose that $R(y)\ge \frac{C_1}{2\sqrt{C_2}} y^{1-\frac{1}{2n}}$ for any $y\le m$, then for $y=m+1$ we have that
\begin{align*}
    R(y)&\ge R(y-x(y))+r(x(y))\\
    &\ge \frac{C_1}{2\sqrt{C_2}} (y-x)^{1-\frac{1}{2n}}+C_1\sqrt{x}\\
    &=\frac{C_1}{2\sqrt{C_2}} y^{1-\frac{1}{2n}}(1-\frac{x}{y})^{1-\frac{1}{2n}}+C_1\sqrt{x}\\
    &\ge \frac{C_1}{2\sqrt{C_2}} y^{1-\frac{1}{2n}}(1-\frac{2(1-\frac{1}{2n})x}{y})+C_1\sqrt{x}\\
    &=\frac{C_1}{2\sqrt{C_2}} y^{1-\frac{1}{2n}}+C_1\sqrt{x}-\frac{C_1}{2\sqrt{C_2}} y^{1-\frac{1}{2n}}\frac{2(1-\frac{1}{2n})x}{y}\\
    &\ge \frac{C_1}{2\sqrt{C_2}} y^{1-\frac{1}{2n}}
\end{align*}
where the inequality $(1-\epsilon)^{\beta}=e^{\beta \log(1-\epsilon)}\ge e^{-2\beta \epsilon}\ge 1-2\beta \epsilon$ is used for $0<\beta<1$, $0<\epsilon\le \frac{1}{2}$.

\section{Proof of Proposition \ref{cor:lb}}
We first verify that the choice of $x(y)=y^{\frac{1-\beta}{1-\alpha}}$ indeed satisfies the two properties.

The first property is now equivalent to proving
$$
\frac{\beta}{2} y^{\frac{\alpha(1-\beta)}{1-\alpha}}\le y^{\beta}-(y-y^{\frac{1-\beta}{1-\alpha}})^{\beta}
$$
We estimate the RHS as follows
\begin{align*}
    y^{\beta}-(y-y^{\frac{1-\beta}{1-\alpha}})^{\beta}&=y^{\beta}-y^{\beta}(1-y^{\frac{\alpha-\beta}{1-\alpha}})^{\beta}\\
    &\ge y^{\beta}-y^{\beta}(1-\frac{\beta y^{\frac{\alpha-\beta}{1-\alpha}}}{2})\\
    &=\frac{\beta y^{\frac{\alpha(1-\beta)}{1-\alpha}}}{2}
\end{align*}

The second property follows from the same reasoning in Section \ref{s4}, that such choice of $x(y)$ corresponds to interval length of form $2^{{(1+\frac{\beta-\alpha}{1-\beta})}^k}$.

The impossibility argument on $x(y)=o(y^{\frac{1-\beta}{1-\alpha}})$ follows from the same analysis of Proposition \ref{prop:lb}, which is actually a special case with $\alpha=\frac{1}{2}$ and $\beta=1-\frac{1}{2n}$.

\section{Proof of Lemma \ref{lem exp}}
The proof is identical to that of Theorem \ref{main}, except that the regret of experts and the recursion are different. The regret of experts are guaranteed to be $O(\log I)$ by using ONS \cite{hazan2007logarithmic} as the expert algorithm $\A$ for exp-concave loss, or by using OGD for strongly-convex loss. It's worth to notice that any $\lambda$-strongly-convex function is also $\frac{\lambda}{G^2}$-exp-concave.

Let us check the covering property of intervals first. Then only difference between Algorithm \ref{alg3} and Algorithm \ref{alg2} is that instead of initiating (potentially) multiple experts with different lifespans at some time $t$, Algorithm \ref{alg3} only initiates the expert with the largest lifespan. As a result, it has no effect on the covering and Lemma \ref{p1 new} still holds. Because both ONS for exp-concave loss and OGD for strongly-convex loss are adaptive to the horizon, the regret on the small interval $[i,t]$ remains the optimal $O(\log T)$.

We only need to solve the recursion when interval length of form $2^{(1+\epsilon)^k}$ is used. By a similar argument to Lemma 3.3 in \cite{hazan2009efficient}, the regret $r(x)$ over the small interval is $O(\log T +\log x)=O(\log T)$ which we discuss later. Recall that this interval length choice corresponds to $y=x^{1+\epsilon}$, and now we are solving
    $$
    R(x^{1+\epsilon}-x)+\log T\le R(x^{1+\epsilon})
    $$
    We claim that $R(x)=\frac{2\log T x^{\epsilon}}{\epsilon}$ is valid, by the following argument. The claim is equal to proving
    $$
    x^{\epsilon(1+\epsilon)}-(x^{1+\epsilon}-x)^{\epsilon}\ge \frac{\epsilon }{2}
    $$
    We have the following estimation on the LHS:   
    \begin{align*}
    x^{\epsilon(1+\epsilon)}-(x^{1+\epsilon}-x)^{\epsilon}&=x^{\epsilon(1+\epsilon)}(1-(1-x^{-\epsilon})^{\epsilon})\\
    &\ge x^{\epsilon(1+\epsilon)}(1-(1-\frac{\epsilon x^{-\epsilon}}{2}))\\
    &=\frac{\epsilon x^{\epsilon^2}}{2}\ge \frac{\epsilon}{2}
    \end{align*}
    for any $x\ge 1$ and $0<\epsilon<1$ which proves the lemma. The first inequality is due to 
    $$
    (1-x^{-\epsilon})^{\epsilon}=e^{\epsilon \log (1-x^{-\epsilon})}\le e^{-\epsilon x^{-\epsilon}}\le 1-\frac{\epsilon x^{-\epsilon}}{2}.
    $$

Now we finish the proof for the argument that the regret $r(x)$ over the small interval is $O(\log T)$. We follow the method of \cite{hazan2009efficient}. The regret $r(x)$ can be decomposed as the regret of the expert algorithm and the regret of the multiplicative weight algorithm against the best expert. The regret of the expert algorithm $\sum_{\tau=s}^t \ell_{\tau}(x_{\tau}^{(s,k)})-\min_{x} \ell_{\tau}(x)$ can be upper bounded by $O(\log T)$ by the regret guarantees of ONS and OGD.

Using the $\alpha$-exp-concavity of $\ell_t$, we have that
$$
e^{-\alpha \ell_t(x_t)}=e^{-\alpha \sum_{(j,k)\in S_t} w_t^{(j,k)} x_t^{(j,k)}}\ge \sum_{(j,k)\in S_t} w_t^{(j,k)} e^{-\alpha \ell_t(x_t^{(j,k)})}
$$
Taking logarithm,
$$
\ell_t(x_t)\le -\frac{1}{\alpha} \log \sum_{(j,k)\in S_t} w_t^{(j,k)} e^{-\alpha \ell_t(x_t^{(j,k)})}
$$
as a result,
\begin{align*}
    \ell_t(x_t)-\ell_t(x_t^{(j,k)})&\le \frac{1}{\alpha}(\log e^{-\alpha \ell_t(x_t^{(j,k)})}-\log \sum_{(j,k)\in S_t} w_t^{(j,k)} e^{-\alpha \ell_t(x_t^{(j,k)})})\\
    &=\frac{1}{\alpha}\log \frac{\hat{w}_{t+1}^{(j,k)}}{w_t^{(j,k)}}
\end{align*}
If $i<t$, we have that 
$$
\ell_t(x_t)-\ell_t(x_t^{(j,k)})=\frac{1}{\alpha}[\log \frac{\hat{w}_{t+1}^{(j,k)}}{\hat{w}_t^{(j,k)}}+\frac{\hat{w}_t^{(j,k)}}{w_t^{(j,k)}}]\le \frac{1}{\alpha}(\log \hat{w}_{t+1}^{(j,k)}-\log \hat{w}_t^{(j,k)}+\frac{2}{t})
$$
For $i=t$ we have that $w_t^{(t,k)}\ge -\log t$, thus
$$
\ell_t(x_t)-\ell_t(x_t^{(t,k)})\le \frac{1}{\alpha}(\log \hat{w}_{t+1}^{(t,k)} +\log t)
$$
Therefore, the regret against the desired expert $\mA_{(s,k)}$ over any interval $[s,t]$ can be bounded by
\begin{align*}
    \sum_{\tau=s}^t \ell_{\tau}(x_{\tau})-\ell_{\tau}(x_{\tau}^{(s,k)})&=\ell_{s}(x_{s})-\ell_{s}(x_{s}^{(s,k)})+\sum_{\tau=s+1}^t \ell_{\tau}(x_{\tau})-\ell_{\tau}(x_{\tau}^{(s,k)})\\
    &\le \frac{1}{\alpha}(\log \hat{w}_{t+1}^{(t,k)} +\log t+\sum_{\tau=s+1}^t \frac{2}{\tau})\\
    &\le \frac{2}{\alpha}(\log I +\log T).
\end{align*}

\end{document}